\documentclass[twoside,11pt]{article}
\usepackage{jmlr2e}

\usepackage[T1]{fontenc}
\usepackage{times,dsfont}
\usepackage{amsmath,amssymb}
\usepackage{algorithm,algorithmic}
\usepackage{graphicx,color}
\usepackage{float}
\usepackage{subfigure}
\usepackage{multirow}
\usepackage{units}

\ShortHeadings{Learning Multi-modal Similarity}{McFee and Lanckriet}

\def\I{\ensuremath{\mathds{1}}}

\def\R{\ensuremath{\mathbb{R}}}
\def\C{\ensuremath{\mathcal{C}}}
\def\K{\ensuremath{\mathcal{K}}}
\def\H{\ensuremath{\mathcal{H}}}

\def\X{\ensuremath{\mathcal{X}}}
\def\one{\ensuremath{\mathbf{1}}}

\def\trans{\ensuremath{^\mathsf{T}}}

\def\th{\ensuremath{^\text{th}}}

\DeclareMathOperator{\trace}{tr}

\DeclareMathOperator{\diag}{diag}

\DeclareMathOperator{\diam}{diam}

\def\ie{\emph{i.e.}}
\def\eg{\emph{e.g.}}

\newtheorem{defn}{Definition}[section]

\begin{document}

\author{\name Brian McFee \email bmcfee@cs.ucsd.edu\\
\addr Department of Computer Science and Engineering\\
University of California\\
San Diego, CA 92093-0404, USA
\AND
\name Gert Lanckriet \email gert@ece.ucsd.edu\\
\addr Department of Electrical and Computer Engineering\\
University of California\\
San Diego, CA 92093-0407, USA
}
\title{Learning Multi-modal Similarity}

\editor{}
\maketitle

\begin{abstract}
In many applications involving multi-media data, the definition of similarity between items is integral to several key tasks, \eg, nearest-neighbor
retrieval, classification, and recommendation.  Data in such regimes typically exhibits multiple modalities, such as acoustic and visual content of video.
Integrating such heterogeneous data to form a holistic similarity space is therefore a key challenge to be overcome in many real-world applications.

We present a novel multiple kernel learning technique for integrating heterogeneous data into a single, unified similarity space.  Our algorithm learns an
optimal ensemble of kernel transformations which conform to measurements of human perceptual similarity, as expressed by relative comparisons.  To cope with
the ubiquitous problems of subjectivity and inconsistency in multi-media similarity, we develop graph-based techniques to filter similarity measurements,
resulting in a simplified and robust training procedure.  

\end{abstract}

\section{Introduction}


In applications such as content-based recommendation systems, the definition of a proper similarity measure between items is crucial to many tasks, 
including nearest-neighbor retrieval and classification.  In some cases, a natural notion of similarity may emerge from domain knowledge, \eg, cosine 
similarity for bag-of-words 
models of text.  However, in more complex, multi-media domains, there is often no obvious choice of similarity measure.  Rather, viewing different aspects 
of the data may lead to several different, and apparently equally valid notions of similarity.  For example, if the corpus consists of musical
data, each song or artist may be represented simultaneously by acoustic features (such as rhythm and timbre), semantic features (tags, lyrics), or 
social features (collaborative filtering, artist reviews and biographies, etc).  Although domain knowledge may be employed to imbue each representation
with an intrinsic geometry --- and, therefore, a sense of similarity --- the different notions of similarity may not be mutually consistent.  In such 
cases, there is generally no obvious way to combine representations to form a unified similarity space which optimally integrates heterogeneous data.




Without extra information to guide the construction of a similarity measure, the situation seems hopeless.  However, if some side-information is available,
\eg, as provided by human labelers, it can be used to formulate a learning algorithm to optimize the similarity measure.

This idea of using side-information to optimize a similarity function has received a great deal of attention in recent years.  Typically, the notion of
similarity is captured by a distance metric over a vector space (\eg, Euclidean distance in $\R^d$), and the problem of optimizing similarity reduces to 
finding a suitable embedding of the data under a specific choice of the distance metric.  \emph{Metric learning} methods, as they are known in the machine 
learning literature, can be informed by various types of side-information, including class labels~\citep{xing03,nca,collapsing,lmnn}, or binary 
\emph{similar}/\emph{dissimilar} pairwise labels~\citep{wagstaff01,rca,bilenko04,globerson07,itml}.  Alternatively, multidimensional scaling (MDS) 
techniques are typically formulated in terms of quantitative (dis)similarity measurements~\citep{torgerson52,kruskal64,cox94,borg05}.  In these settings, 
the representation of data is optimized so that distance (typically Euclidean) conforms to side-information.  Once a suitable metric has been learned,
similarity to new, unseen data can be computed either directly (if the metric takes a certain parametric form, \eg, a linear projection matrix), or via 
out-of-sample extensions~\citep{bengio04}.

To guide the construction of a similarity space for multi-modal data, we adopt the idea of using similarity measurements, provided by human labelers, as side-information. However, 
it has to be noted that, especially in heterogeneous, multi-media domains, similarity may itself be a highly subjective concept and 
vary from one labeler to the next~\citep{ellis02}.  Moreover, a single labeler may not be able to consistently decide if or to what extent two 
objects are similar, but she may still be able to reliably produce a rank-ordering of similarity over pairs~\citep{kendall90}.
Thus, rather than rely on quantitative similarity or hard binary labels of pairwise similarity, it is now becoming increasingly common to collect similarity information
in the form of triadic or \emph{relative} comparisons~\citep{schultz04,gnmds}, in which human labelers answer questions of the form:
\begin{center}
``Is $x$ more similar to $y$ or $z$?''
\end{center}
Although this form of similarity measurement has been observed to be more stable than quantitative similarity~\citep{kendall90}, and clearly provides a 
richer representation than binary pairwise similarities, it is still subject to problems of consistency and inter-labeler agreement.  It is therefore 
imperative that great care be taken to ensure some sense of robustness when working with perceptual similarity measurements.

In the present work, our goal is to develop a framework for integrating multi-modal data so as to optimally conform to perceptual similarity encoded by relative comparisons. In particular, we follow three guiding principles in the development of our framework:
\begin{enumerate}
\item The embedding algorithm should be robust against subjectivity and inter-labeler disagreement.
\item The algorithm must be able to integrate multi-modal data in an optimal way, \ie, the distances between embedded points should conform to perceptual similarity measurements.
\item It must be possible to compute distances to new, unseen data as it becomes available.
\end{enumerate}


We formulate this problem of heterogeneous feature integration as a learning problem: given a data set, and a collection of relative comparisons 
between pairs, learn a representation of the data that optimally reproduces the similarity measurements.  This type of embedding problem has been previously
studied by \cite{gnmds} and~\cite{schultz04}.  However, \cite{gnmds} provide no out-of-sample extension, and neither support heterogeneous feature
integration, nor do they address the problem of noisy similarity measurements.

A common approach to optimally integrate heterogeneous data is based on \emph{multiple kernel learning}, where each kernel encodes a different modality 
of the data. Heterogeneous feature integration via multiple kernel learning has been addressed by previous authors in a variety of contexts, including
classification~\citep{lanckriet04,zien07,kloft09,nath09}, regression~\citep{sonnenburg06,bach08,cortes09}, and dimensionality reduction~\citep{lin09}.  
However, none of these methods specifically address the problem of learning a unified data representation which conforms to perceptual similarity 
measurements.\\


\subsection{Contributions}

Our contributions in this work are two-fold.  First, we develop the \emph{partial order embedding} (POE) framework~\citep{mcfee09_mkpoe}, which allows us to use graph-theoretic algorithms to filter a collection of subjective similarity measurements for consistency and redundancy.  We then formulate a novel multiple kernel 
learning (MKL) algorithm which learns an ensemble of feature space projections to produce a unified similarity space.  Our method is able to produce 
non-linear embedding functions which generalize to unseen, out-of-sample data.  Figure~\ref{fig:overview} provides a high-level overview of the proposed 
methods.

\begin{figure}
\begin{center}
\includegraphics[width=0.65\textwidth]{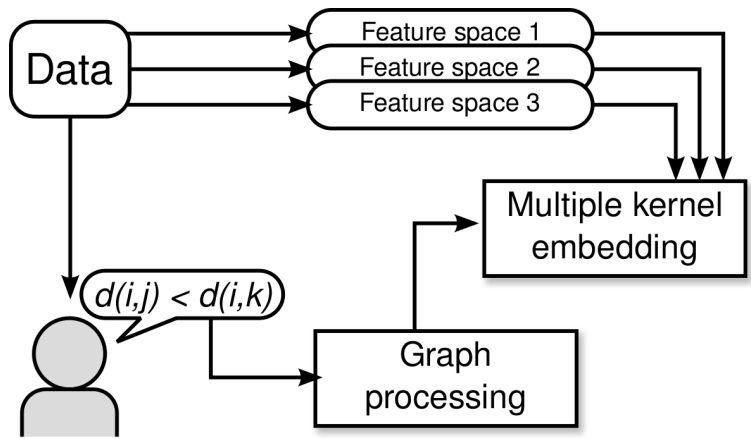}%
\end{center}
\caption{An overview of our proposed framework for multi-modal feature integration.  Data is represented in multiple feature spaces (each encoded by a
kernel function).  Humans supply perceptual similarity measurements in the form of relative pairwise comparisons, which are in turn filtered by graph 
processing algorithms, and then used as constraints to optimize the multiple kernel embedding.\label{fig:overview}}
\end{figure}

The remainder of this paper is structured as follows.  In Section~\ref{sec:poe}, we develop a graphical framework for interpreting and manipulating 
subjective similarity measurements.  
In Section~\ref{sec:parametric}, we derive an embedding algorithm which learns an optimal transformation of a single feature space.
In Section~\ref{sec:mkl}, we 
develop a novel multiple-kernel learning formulation for embedding problems, and derive 
an algorithm to learn an optimal space from heterogeneous data.
Section~\ref{sec:exp} provides experimental results illustrating the effects of graph-processing on noisy similarity data, and the effectiveness of 
the multiple-kernel embedding algorithm on a music similarity task with human perception measurements.
Finally, we prove hardness of dimensionality reduction in this setting in Section~\ref{sec:nphard}, and conclude in
Section~\ref{sec:conclusion}.

\subsection{Preliminaries}


A \emph{(strict) partial order} is a binary relation $R$ over a set $Z$ ($R{\subseteq}Z^2$) which satisfies the following properties:\footnote{The standard definition
of a (non-strict) partial order also includes the \emph{reflexive} property: $\forall a, (a,a) \in \R$.  For reasons that will become clear in 
Section~\ref{sec:poe}, we take the \emph{strict} definition here, and omit the reflexive property.}
\begin{itemize}
\item Irreflexivity: $(a,a) \notin R$,
\item Transitivity: $(a,b) \in R \wedge (b,c) \in R \Rightarrow (a,c) \in R$,
\item Anti-symmetry: $(a,b) \in R \Rightarrow (b,a) \notin R$.
\end{itemize}

Every partial order can be equivalently represented as a directed acyclic graph (DAG), where each vertex is an element of $Z$ and an edge is drawn from 
$a$ to $b$ if $(a,b)\in R$.  
For any partial order, $R$ may refer to either the set of ordered tuples $\{(a,b)\}$ or the graph (DAG) representation of the partial order; the use will be
clear from context.
Let $\diam(R)$ denote the length of the longest (finite) source-to-sink path in the graph of $R$.

For a directed graph $G$, we denote by $G^\infty$ its \emph{transitive closure}, \ie, $G^\infty$ contains an edge $(i,j)$ if and only if there exists a 
path from $i$ to $j$ in $G$.  Similarly, the \emph{transitive reduction} (denoted $G^{\min}$) is the minimal graph with equivalent transitivity to $G$, \ie, the 
graph with the fewest edges such that $\left(G^{\min}\right)^\infty = G^\infty$.

Let $\X = \{x_1, x_2,\dots,x_n\}$ denote the training set of $n$ items.  A \emph{Euclidean embedding} is a function $g:\X\rightarrow\R^d$ which maps $\X$ into a 
$d$-dimensional space equipped with the Euclidean ($\ell_2$) metric:
\[
 \|x-y\|_2 = \sqrt{(x-y)\trans(x-y)}.
\]

For any matrix $B$, let $B_i$ denote its $i\th$ column vector.   A symmetric matrix $A \in \R^{n\times n}$ has 
a spectral decomposition $A=V{\Lambda}V\trans$, where $\Lambda = \diag(\lambda_1, \lambda_2,\dots,\lambda_n)$ is a diagonal matrix containing the
eigenvalues of $A$, and $V$ contains the eigenvectors of $A$.  We adopt the convention that eigenvalues (and corresponding
eigenvectors) are sorted in descending order.  $A$ is 
\emph{positive semi-definite}~(PSD), denoted by $A\succeq0$, if each eigenvalue is non-negative: $\lambda_i\geq0, ~i = 1, \ldots, n$.  Finally, a PSD matrix $A$ gives rise to the Mahalanobis distance function
\[
\|x-y\|_A = \sqrt{(x-y)\trans A(x-y)}.
\]

\section{A graphical view of similarity}
\label{sec:poe}
Before we can construct an embedding algorithm for multi-modal data, we must first establish the form of side-information that will drive the algorithm, i.e., the similarity measurements that will be collected from human labelers.  There is an extensive body of work on the
topic of constructing a geometric representation of data to fit perceptual similarity measurements.  Primarily, this work falls under the umbrella of
multi-dimensional scaling (MDS), in which perceptual similarity is modeled by numerical responses corresponding to the perceived ``distance'' between a pair
of items, \eg, on a similarity scale of 1--10.  (See~\cite{cox94,borg05} for comprehensive overviews of MDS techniques.)

Because ``distances'' supplied by test subjects may not satisfy metric properties --- in particular, they may not correspond to Euclidean distances --- 
alternative \emph{non-metric} MDS (NMDS) techniques have been proposed~\citep{kruskal64}.  Unlike classical or metric MDS techniques, which seek to preserve
quantitative distances, NDMS seeks an embedding in which the rank-ordering of distances is preserved.  

Since NMDS only needs the rank-ordering of distances, and not the distances themselves, the task of collecting similarity measurements can be simplifed by
asking test subjects to \emph{order pairs of points by similarity}:
\begin{center}
``Are $i$ and $j$ more similar than $k$ and $\ell$?''
\end{center}
or, as a special case, the ``triadic comparison''
\begin{center}
``Is $i$ more similar to $j$ or $\ell$?''
\end{center}
Based on this kind of \emph{relative comparison} data, the embedding problem can be formulated as follows. Given is a set of objects $\X$, and a set of similarity measurements $\C = \{(i,j,k,\ell)\} \subseteq\X^4$, where a tuple $(i,j,k,\ell)$ is interpreted as ``$i$ and $j$ are more similar than $k$ and $\ell$.''  (This formulation subsumes the triadic comparisons model when $i=k$.) The goal is to find an embedding function $g:\X\rightarrow\R^d$ such that
\begin{equation}
\forall (i,j,k,\ell) \in \C: ~ \|g(i) - g(j)\|^2 + 1 < \|g(k) - g(\ell)\|^2. \label{eq:constraints}
\end{equation}
The unit margin is forced between the constrained distances for numerical stability.

\cite{gnmds} work with this kind of relative comparison data and describe
a generalized NMDS algorithm (GNMDS), which formulates the embedding problem as a semi-definite program.  \cite{schultz04} derive a similar algorithm which
solves a quadratic program to learn a linear, axis-aligned transformation of data to fit relative comparisons.

Previous work on relative comparison data often treats each measurement $(i,j,k,\ell)\in\C$ as effectively independent~\citep{schultz04,gnmds}.  However, 
due to their semantic interpretation as encoding pairwise similarity comparisons, and the fact that a pair $(i,j)$ may participate in several 
comparisons with other pairs, 
there may be some \emph{global} structure to $\C$ which these previous methods are unable to exploit.

In Section \ref{sec:similarity}, we develop a graphical framework to infer and interpret the global structure exhibited by the constraints of the embedding problem. Graph-theoretic algorithms presented in Section \ref{sec:simplification} then exploit this representation to filter this collection of noisy similarity measurements for consistency and redundancy. The final, reduced set of relative comparison constraints defines a partial order, making for a more robust and efficient embedding problem.



\subsection{Similarity graphs}
\label{sec:similarity}

\begin{figure}
\centering
\hfill\includegraphics[width=0.25\textwidth]{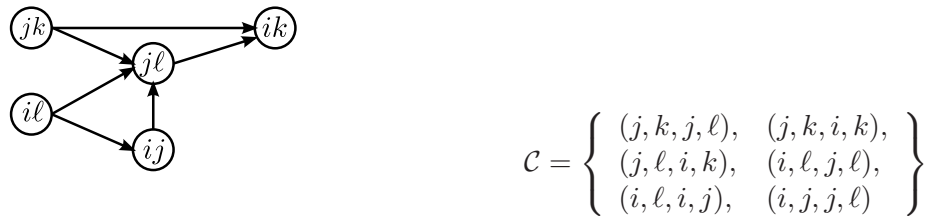}%
\hfill$\C =\left\{\begin{array}{ll}
(j,k,j,\ell),& (j,k,i,k),\\
(j,\ell,i,k),& (i,\ell,j,\ell),\\
(i,\ell,i,j),& (i,j,j,\ell)
\end{array}\right\}$\hfill\vfill
\caption{The graph representation (left) of a set of relative comparisons (right).\label{fig:constraintdag}}
\end{figure}

To gain more insight into the underlying structure of a collection of comparisons $\C$, we can represent $\C$ as a directed graph over $\X^2$.  Each 
vertex in the graph corresponds to a pair $(i,j)\in \X^2$, and an edge from $(i,j)$ to $(k,\ell)$ corresponds to a similarity measurement 
$(i,j,k,\ell)$ (see Figure~\ref{fig:constraintdag}).  Interpreting $\C$ as a graph will allow us to infer properties of \emph{global} (graphical) 
structure of $\C$.  In particular, two facts become immediately apparent: 
\begin{enumerate}
\item If $\C$ contains cycles, then there exists no embedding which can satisfy $\C$.  
\item If $\C$ is acyclic, any embedding that satisfies the transitive reduction $\C^{\min}$ also satisfies $\C$.
\end{enumerate}

The first fact implies that no algorithm can produce an embedding which satisfies all measurements if the graph is cyclic.  
In fact, the converse of this statement is also true: if $\C$ is acyclic, then an embedding exists in which all similarity 
measurements are preserved (see Appendix~\ref{sec:diameter}).
If $\C$ is cyclic, however, by analyzing the graph, it is possible to identify an ``unlearnable'' subset of $\C$ which must be violated by any embedding.

Similarly, the second fact exploits the transitive nature of distance comparisons.  In the example depicted in Figure~\ref{fig:constraintdag}, any $g$ that
satisfies $(j,k,j,\ell)$ and $(j,\ell,i,k)$ must also satisfy $(j,k,i,k)$.  In effect, the constraint $(j,k,i,k)$ is redundant, and may also be safely
omitted from $\C$.

These two observations allude to two desirable properties in $\C$ for embedding methods: \emph{transitivity} and \emph{anti-symmetry}.  Together with irreflexivity, these fit the defining characteristics of a \emph{partial order}. Due to subjectivity and inter-labeler disagreement, however, most collections of relative comparisons will not define a partial order. Some graph processing, presented next, based on an approximate maximum acyclic subgraph algorithm, can  reduce them to a partial order.

\subsection{Graph simplification}
\label{sec:simplification}

Because a set of similarity measurements $\C$ containing cycles cannot be embedded in any Euclidean space, $\C$ is inherently inconsistent.
Cycles in $\C$ therefore constitute a form of \emph{label noise}.  As noted by~\cite{angelova04}, label noise can have adverse effects on both model
complexity and generalization.  This problem can be mitigated by detecting and pruning noisy (confusing) examples, and training on a reduced, but 
certifiably ``clean'' set~\citep{angelova05, vezhnevets07}.
Unlike most settings, where the noise process affects each label independently --- \eg, random classification noise~\citep{angluin88} --- the graphical 
structure of interrelated relative comparisons can be exploited to detect and prune inconsistent measurements.  By eliminating similarity measurements which 
cannot be realized by any embedding, the optimization procedure can be carried out more efficiently and reliably on a reduced constraint set.


Ideally, when eliminating edges from the graph, we would like to retain as much information as possible.  Unfortunately, this is equivalent to the
\emph{maximum acyclic subgraph} problem, which is NP-Complete~\citep{garey79}. 
A $\nicefrac{1}{2}$-approximate solution can be achieved by a simple greedy algorithm (Algorithm~\ref{alg:amas})~\citep{berger90}.

\begin{algorithm}
\caption{Approximate maximum acyclic subgraph}
\label{alg:amas}
\begin{algorithmic}
    \STATE{{\bfseries Input}: Directed graph $G = (V,E)$}
    \STATE{{\bfseries Output}: Acyclic graph $G'$}
    \STATE{$E' \leftarrow \emptyset$}
    \FOR{each $(u,v) \in E$ in random order}
        \IF{$E' \cup \{(u,v)\}$ is acyclic}
            \STATE{$E' \leftarrow E' \cup \{(u,v)\}$}
        \ENDIF%
    \ENDFOR%
    \STATE{$G' \leftarrow (V, E')$}
\end{algorithmic}
\end{algorithm}

Once a consistent subset of similarity measurements has been produced, it can be simplified further by pruning redundancies.  In the graph view of
similarity measurements, redundancies can be easily removed by computing the transitive reduction of the graph~\citep{aho:131}.

By filtering the constraint set for
consistency, we ensure that embedding algorithms are not learning from spurious information.  Additionally, pruning the constraint set by transitive
reduction focuses embedding algorithms on the most important core set of constraints while reducing overhead due to redundant information.


\section{Partial order embedding}

\label{sec:parametric}
%

%
Now that we have developed a language for expressing similarity between items, we are ready to formulate the embedding problem.
In this section, we develop an algorithm that learns a representation of data consistent with a collection of relative similarity measurements, and 
allows to map unseen data into the learned similarity space after learning.  In order to accomplish this, we will assume a feature representation for $\X$.  By parameterizing the embedding function $g$ in terms
of the feature representation, we will be able to apply $g$ to any point in the feature space, thereby generalizing to data outside of the training set.


\subsection{Linear projection}
To start, we assume that the data originally lies in some Euclidean space, \ie, $\X\subset \R^D$.
There are of course many ways to define an 
embedding function $g:\R^D\rightarrow\R^d$.  Here, we will restrict attention to embeddings parameterized by a linear projection matrix $M$, so 
that for a vector $x\in\R^D$, 
\[
g(x)\doteq Mx.
\]
Collecting the vector representations of the training set as columns of a matrix $X\in\R^{D\times n}$, the inner product matrix of 
the embedded points can be characterized as
\begin{equation}
A = X\trans M\trans M X.\label{eq:innerproduct}
\end{equation}

Now, for a relative comparison $(i,j,k,\ell)$, we can express the distance constraint (\ref{eq:constraints}) between embedded points as follows:
\begin{equation}
(X_i - X_j)\trans M\trans M (X_i - X_j) + 1 \leq (X_k - X_\ell)\trans M\trans M (X_k - X_\ell).\label{eq:constraintquadratic}
\end{equation}
These inequalities can then be used to form the constraint set of an optimization problem to solve for $M$.  Because, in general, $\C$ may not be 
satisfiable by a linear projection of $\X$, we soften the constraints by introducing a slack variable $\xi_{ijk\ell}\geq0$ for each constraint, 
and minimize the empirical hinge loss over constraint violations $\nicefrac{1}{|\C|}\sum_\C\xi_{ijk\ell}$.
This choice of loss function can be interpreted as a 
generalization of ROC area (see Appendix~\ref{sec:gauc}).  

To avoid over-fitting, we introduce a regularization term $\trace(M\trans M)$, and a trade-off parameter $\beta>0$ to control the balance 
between regularization and loss minimization.  This leads to a regularized risk minimization objective:
\begin{eqnarray}
\min_{M,\xi\geq0} & & \trace(M\trans M) + \frac{\beta}{|\C|}\sum_\C\xi_{ijk\ell} \label{opt:lin-M} \\
\text{s.t.} & & (X_i - X_j)\trans M\trans M (X_i - X_j) + 1 \leq (X_k - X_\ell)\trans M\trans M (X_k - X_\ell) + \xi_{ijk\ell}, \nonumber \\ & & \forall (i,j,k,\ell) \in \C. \nonumber
\end{eqnarray}
After learning $M$ by solving this optimization problem, the embedding can be extended to out-of-sample points $x'$ by applying the projection: $x' \mapsto Mx'$.

Note 
that the distance constraints in (\ref{opt:lin-M}) involve differences of 
quadratic terms, and are therefore not convex.  However, since $M$ only appears in the form $M\trans M$ in (\ref{opt:lin-M}), we can equivalently express the optimization problem in terms of a 
positive semi-definite matrix $W\doteq M\trans M$. This change of variables results in Algorithm~\ref{alg:lpoe}, a \emph{convex} optimization problem, more specifically a semi-definite programming (SDP) problem~\citep{boyd04}, since objective and constraints are linear in $W$, including the linear matrix inequality $W \succeq 0$. The corresponding inner product matrix is
\[
A = X\trans W X.
\]

Finally, after 
the optimal $W$ is found, the embedding function $g$ can be recovered from the spectral decomposition of $W$:
\[
W = V\Lambda V\trans \quad \Rightarrow \quad g(x) = \Lambda^{1/2} V\trans x.
\]

\begin{algorithm}
\caption{Linear partial order embedding (LPOE)}\label{alg:lpoe}%
\begin{algorithmic}
    \STATE {\bfseries Input:} $n$ objects $\X$,\\partial order $\C$,\\data matrix $X\in\R^{D\times n}$,\\$\beta> 0$
    \STATE {\bfseries Output:} mapping $g:\X\rightarrow \R^{d}$
    \begin{align*}
        \min_{W,\xi} ~~~ &\hfill \trace\left(W\right) + \frac{\beta}{|\C|} \sum_{\C} \xi_{ijk\ell}\\
        d(x_i,x_j) &\doteq \left( X_i - X_j\right)\trans W \left(X_i - X_j\right)\\
        d(x_i,x_j) + 1 & \leq d(x_k,x_\ell) + \xi_{ijk\ell}\\
        \xi_{ijk\ell} &\geq 0
            &\forall (i,j,k,\ell) \in \C\\
        W & \succeq 0\\
    \end{align*}
\end{algorithmic}
\end{algorithm}

\subsection{Non-linear projection via kernels}

The formulation in Algorithm~\ref{alg:lpoe} can be generalized to support non-linear embeddings by the use of kernels, following the method of~\cite{globerson07}: we first map the
data into a reproducing kernel Hilbert space (RKHS) $\H$ via a feature map $\phi$ with corresponding kernel function $k(x,y) = \langle \phi(x), \phi(y) \rangle_\H$; then, the data is mapped to $\R^d$ by a linear projection 
$M:\H\rightarrow\R^d$.  The embedding function $g:\X\rightarrow\R^d$ is the therefore the composition of the projection $M$ with $\phi$:
\[
g(x) = M(\phi(x)).
\]
Because $\phi$ may be non-linear, this allows to learn a non-linear embedding $g$.

More precisely, we consider $M$ as being comprised of $d$ elements of $\H$, \ie, $\{\omega_1,\omega_2,\dots,\omega_d\} \subseteq \H$.  The embedding $g$ can thus be 
expressed as
\[
g(x) = \left(\langle \omega_p, \phi(x)\rangle_\H \right)_{p=1}^d,
\]
where $(\cdot)_{p=1}^d$ denotes concatenation over $d$ vectors.

Note that in general, $\H$ may be infinite-dimensional, so directly optimizing $M$ may not be feasible.  However, by appropriately regularizing 
$M$, we may invoke the generalized representer theorem~\citep{representer}.  Our choice of regularization is the Hilbert-Schmidt norm of $M$, which,
in this case, reduces to
\[
\|M\|_{\text{HS}}^2 = \sum_{p=1}^d \langle \omega_p, \omega_p \rangle_\H.
\]
With this choice of regularization, it follows from the generalized representer theorem that at an optimum, each $\omega_p$ must lie in the span of the training data, \ie,
\[
\omega_p = \sum_{i=1}^n N_{pi} \phi(x_i), \quad p = 1, \ldots, d,
\]
for some real-valued matrix $N \in \R^{d \times n}$.  If $\Phi$ is a matrix representation of $\X$ in $\H$ (\ie, $\Phi_i=\phi(x_i)$ for $x_i\in\X$), then the projection operator
$M$ can be expressed as
\begin{equation}\label{eq:M-N}
M = N\Phi\trans.
\end{equation}

We can now reformulate the embedding problem as an optimization over $N$ rather than $M$. Using (\ref{eq:M-N}), the regularization term can be expressed as
\[
\| M \|_{\text{HS}}^2 = \trace(\Phi N\trans N \Phi\trans) = \trace(N\trans N \Phi\trans \Phi) = \trace(N\trans NK),
\]
where $K$ is the kernel matrix over $\X$:
\[
K = \Phi\trans \Phi, \quad \text{with}~~K_{ij} = \langle \phi(x_i), \phi(x_j) \rangle_\H = k(x_i,x_j).
\]
To formulate the distance constraints in terms of $N$, we first express the embedding $g$ in terms of $N$ and the kernel function:
\[
g(x) = M(\phi(x)) = N \Phi^T (\phi(x)) = N\left(\langle\Phi_i,\phi(x)\rangle_\H\right)_{i=1}^n = N\left(k(x_i,x)\right)_{i=1}^n = NK_x,
\]
where $K_x$ is the column vector formed by evaluating the kernel function $k$ at $x$ against the training set.
The inner product matrix of embedded points can therefore be expressed as
\[
A = K N\trans N K,
\]
which allows to express the distance constraints in terms of $N$ and the kernel matrix $K$:
\[
(K_i - K_j)\trans N\trans N (K_i - K_j) + 1 \leq (K_k - K_\ell)\trans N\trans N (K_k - K_\ell).
\]
The embedding problem thus amounts to solving the following optimization problem in $N$ and $\xi$:
\begin{eqnarray}
\min_{N,\xi\geq0} & & \trace(N\trans N K) + \frac{\beta}{|\C|}\sum_\C\xi_{ijk\ell} \label{opt:nonlin-N} \\
\text{s.t.} & & (K_i - K_j)\trans N\trans N (K_i - K_j) + 1 \leq (K_k - K_\ell)\trans N\trans N (K_k - K_\ell) + \xi_{ijk\ell}, \nonumber \\ & & \forall (i,j,k,\ell) \in \C. \nonumber
\end{eqnarray}

Again, the distance constraints in (\ref{opt:nonlin-N}) are non-convex due to the differences of quadratic terms. 
And, as in the previous section, $N$ only appears in the form of inner products $N\trans N$ in (\ref{opt:nonlin-N}) --- both in the constraints, and in the regularization term --- so we can 
again derive a convex optimization problem by changing variables to $W\doteq N{\trans}N \succeq 0$. The resulting embedding problem is listed as Algorithm~\ref{alg:kpoe}, again a semi-definite programming problem (SDP), with an objective function and constraints that are linear in $W$.



After solving for $W$, the 
matrix $N$ can be recovered by computing the spectral decomposition $W=V\Lambda V\trans$, and defining $N=\Lambda^{1/2}V\trans$.  The resulting embedding function takes the form:
\[
g(x)=\Lambda^{1/2}V\trans K_x.
\]

As in~\cite{schultz04}, this formulation can be interpreted as learning a Mahalanobis distance metric $\Phi W \Phi\trans$ over $\H$.
More generally, we can view this as a form of kernel learning, where the kernel matrix $A$ is restricted to the set
\begin{equation}
A \in \left\{ KWK ~:~ W\succeq 0 \right\}.\label{eq:kpoerange}
\end{equation}

\begin{algorithm}
\caption{Kernel partial order embedding (KPOE)}\label{alg:kpoe}%
\begin{algorithmic}
    \STATE {\bfseries Input:} $n$ objects $\X$,\\partial order $\C$,\\kernel matrix $K$,\\$\beta> 0$
    \STATE {\bfseries Output:} mapping $g:\X\rightarrow \R^{n}$
    \begin{align*}
        \min_{W,\xi} ~~~ &\hfill \trace\left(WK\right) + \frac{\beta}{|\C|} \sum_{\C} \xi_{ijk\ell}\\
        d(x_i,x_j) &\doteq \left( K_i - K_j\right)\trans W \left(K_i - K_j\right)\\
        d(x_i,x_j) + 1 & \leq d(x_k,x_\ell) + \xi_{ijk\ell}\\
        \xi_{ijk\ell} &\geq 0
            &\forall (i,j,k,\ell) \in \C\\
        W & \succeq 0\\
    \end{align*}
\end{algorithmic}
\end{algorithm}

\subsection{Connection to GNMDS}
\label{sec:parametric:gnmds}
We conclude this section by drawing a connection between Algorithm~\ref{alg:kpoe} and the generalized non-metric MDS (GNMDS) algorithm of~\cite{gnmds}.

First, we observe that the $i$-th column, $K_i$, of the kernel matrix $K$ can be expressed in terms of $K$ and the $i\th$ standard basis vector $e_i$:
\[
K_i = Ke_i.
\]
From this, it follows that distance computations in Algorithm~\ref{alg:kpoe} can be equivalently expressed as
\begin{align}
d(x_i,x_j) &= (K_i - K_j)\trans W (K_i - K_j)\nonumber\\
&= (K(e_i - e_j))\trans W (K(e_i - e_j))\nonumber\\
&= (e_i - e_j)\trans K\trans W K  (e_i - e_j).\label{eq:kernelbasis}
\end{align}
If we consider the extremal case where $K = I$, \ie, we have no prior feature-based knowledge of similarity between points, then
Equation~\ref{eq:kernelbasis} simplifies to
\[
d(x_i,x_j) = (e_i - e_j)\trans IWI (e_i - e_j) = W_{ii} + W_{jj} - W_{ij} - W_{ji}.
\]
Therefore, in this setting, 
rather than defining a feature transformation,
$W$ directly encodes the
inner products between embedded training points.  Similarly, the regularization term becomes
\[
\trace(WK) = \trace(WI) = \trace(W).
\]
Minimizing the regularization term can be interpreted as minimizing 
a convex upper bound on the rank of $W$~\citep{boyd04}, which expresses a preference for low-dimensional embeddings.
Thus, by setting $K=I$ in Algorithm~\ref{alg:kpoe}, we directly recover the GNMDS algorithm.

Note that directly learning inner products between embedded training data points rather than a feature transformation does not allow a meaningful out-of-sample extension, to embed unseen data points.  On the other hand, by Equation~\ref{eq:kpoerange}, it is clear that the 
algorithm optimizes over the entire cone of PSD matrices.  Thus, if $\C$ defines a DAG, we could exploit the fact that a partial order over distances always allows an embedding which satisfies all constraints in $\C$ (see Appendix~\ref{sec:diameter}) to eliminate the slack variables from the program entirely.

\section{Multiple kernel embedding}
\label{sec:mkl}


In the previous section, we derived an algorithm to learn an optimal projection from a kernel space $\H$ to $\R^d$ such that Euclidean distance between 
embedded points conforms to perceptual similarity.  If, however, the data is heterogeneous in nature, it may not be realistic to assume that a single
feature representation can sufficiently capture the inherent structure in the data.  For example, if the objects in question are images, it may be natural to
encode texture information by one set of features, and color in another, and it is not immediately clear how to reconcile these two disparate sources of 
information into a single kernel space.  

However, by encoding each source of information independently by separate feature spaces $\H^1,\H^2,\dots$ --- equivalently, kernel matrices $K^1,K^2,\dots$
--- we can formulate a multiple kernel learning algorithm to optimally combine all feature spaces into a single, unified embedding space.  In this section,
we will derive a novel, projection-based approach to multiple-kernel learning and extend Algorithm~\ref{alg:kpoe} to support heterogeneous data in a
principled way.

\subsection{Unweighted combination}
Let $K^{1}, K^{2},\dots,K^{m}$ be a set of kernel matrices, each with a corresponding feature map $\phi^{p}$ and RKHS $\H^p$, for $p\in1,\dots,m$.  
One natural way to combine the kernels is to look at the product space, which is formed by concatenating the feature maps:
\[
\phi(x_i) = (\phi^1(x_i), \phi^2(x_i), \dots, \phi^m(x_i)) = {\left(\phi^{p}(x_i)\right)}_{p=1}^{m}.
\]
Inner products can be computed in this space by summing across each feature map:
\[
\langle \phi(x_i), \phi(x_j) \rangle = \sum_{p=1}^{m} \left\langle \phi^{p}(x_i),
\phi^{p}(x_j)\right\rangle_{\H^p}.
\]
resulting in the {\em sum-kernel\/} --- also known as the \emph{average kernel} or \emph{product space kernel}.  The corresponding kernel matrix can be conveniently represented 
as the unweighted sum of the base kernel matrices: 
\begin{equation}
\widehat{K} = \sum_{p=1}^m K^{p}. \label{eq:sumkernel}
\end{equation}

Since $\widehat{K}$ is a valid kernel matrix itself, we could use $\widehat{K}$ as input for Algorithm~\ref{alg:kpoe}.  As a result, the algorithm would 
learn a kernel from the family
\begin{align*}
\K_1 &= \left\{ \left( \sum_{p=1}^m K^p \right) W \left( \sum_{p=1}^m K^p \right) ~:~ W \succeq 0\right\}\\
            &= \left\{ \sum_{p,q=1}^m K^p W K^q ~:~ W \succeq 0\right\}.
\end{align*}

\subsection{Weighted combination}
Note that $\K_1$ treats each kernel equally; it is therefore impossible to distinguish \emph{good} features (\ie, those which can be transformed to best fit
$\C$) from \emph{bad} features, and as a result, the quality of the resulting embedding may be degraded.
To combat this phenomenon, it is common to learn a scheme for weighting the kernels in a way which is optimal for a particular task.  The most common 
approach to combining the base kernels is to take a positive-weighted sum
\[
\sum_{p=1}^m \mu_p K^{p}\hspace{4em}(\mu_p\geq 0),
\]
where the weights $\mu_p$ are learned in conjunction with a predictor~\citep{lanckriet04,sonnenburg06,bach08,cortes09}.  Equivalently, this can be viewed as learning a feature map
\[
\phi(x_i) = \left(\sqrt{\mu_p} \phi^{p}(x_i)\right)_{p=1}^m,
\]
where each base feature map has been scaled by the corresponding weight $\sqrt{\mu_p}$.

Applying this reasoning to 
learning an embedding that conforms to perceptual similarity, one might consider a two-stage approach to parameterizing the embedding (Figure~\ref{fig:mkl:embedding:early}): first construct a weighted kernel combination, and then project from the combined kernel space.  \cite{lin09} formulate a dimensionality reduction algorithm in this way.  In the present setting, this would be achieved by simultaneously optimizing $W$ and $\mu_p$ to choose an inner product matrix $A$ from the set
\begin{align}
\K_2 &= \left\{ \left(\sum_{p=1}^m \mu_p K^{p}\right) W \left( \sum_{p=1}^m \mu_p K^{p}\right) ~:~ W\succeq 0, \forall p,~\mu_p\geq 0\right\}\nonumber\\
&= \left\{  \sum_{p,q=1}^m \mu_p K^{p} W \mu_q K^{q} ~:~ W\succeq 0, \forall p,~\mu_p \geq 0\right\}.\label{kernelclasscross}
\end{align}
The corresponding distance constraints, however, contain differences of terms cubic in the optimization variables $W$ and $\mu_p$:
\[
\sum_{p,q} \left(K^p_i - K^p_j \right)\trans \mu_p W \mu_q \left(K^q_i - K^q_j \right) + 1 \leq
\sum_{p,q} \left(K^p_k - K^p_\ell \right)\trans \mu_p W \mu_q \left(K^q_k - K^q_\ell \right), 
\]
and are therefore non-convex and difficult to optimize.  Even simplifying the class by removing cross-terms, \ie, restricting $A$ to the form
\begin{equation}
\K_3 = \left\{  \sum_{p=1}^m \mu_p^2 K^{p} W K^{p} ~:~ W\succeq 0, \forall p,~ \mu_p \geq 0\right\},\label{kernelclassnocross}
\end{equation}
still leads to a non-convex problem, due to the difference of positive quadratic terms introduced by distance 
calculations: 
\[
\sum_{p=1}^m \left(K^p_i - K^p_j \right)\trans \mu_p^2 W \left( \mu_p K^p_i - K^p_j \right) + 1 \leq
\sum_{p=1}^m \left(K^p_k - K^p_\ell \right)\trans \mu_p^2 W \left( \mu_p K^p_k - K^p_\ell \right).
\]
However, a more subtle problem with this formulation lies in the assumption 
that a single weight can characterize the contribution of a kernel to the optimal embedding.  In general, different kernels 
may be more or less informative on different subsets of $\X$ or different regions of the corresponding 
feature space. Constraining the embedding to a single metric $W$ with a single weight $\mu_p$ for each kernel may be too restrictive to take 
advantage of this phenomenon.

\subsection{Concatenated projection}
We now return to the original intuition behind Equation~\ref{eq:sumkernel}.  The sum-kernel represents the inner product between points in the space formed
by concatenating the base feature maps $\phi^p$.  The sets $\K_2$ and $\K_3$ characterize projections of the weighted combination space, and turn out to not
be amenable to efficient optimization (Figure~\ref{fig:mkl:embedding:early}).  This can be seen as a consequence of prematurely combining kernels prior to projection.

Rather than projecting the (weighted) concatenation of $\phi^p(\cdot)$, we could alternatively concatenate learned projections $M^p(\phi^p(\cdot))$, as 
illustrated by Figure~\ref{fig:mkl:embedding:late}.  
Intuitively, by defining the embedding as the concatenation of $m$ different projections, we allow the algorithm to learn an ensemble of projections, each tailored to its corresponding domain space and jointly optimized to produce an optimal space.
By contrast, the previously discussed formulations apply essentially the same projection to 
each (weighted) feature space, and are thus much less flexible than our proposed approach.
Mathematically, an embedding function of this form can be expressed as the concatenation
\[
g(x) = \left( M^p \left(\phi^p(x)\right) \right)_{p=1}^m.
\]

Now, given this characterization of the embedding function, we can adapt Algorithm~\ref{alg:kpoe} to optimize over multiple kernels. As in the single-kernel case, we introduce regularization terms for each projection operator $M^p$
\[
\sum_{p=1}^m \|M^p\|_{\text{HS}}^2
\]
to the objective function.  Again, by invoking the representer theorem for each $M^p$, it follows that 
\[
M^p = N^p \left(\Phi^p\right)\trans,
\]
for some matrix $N^p$, which allows to reformulate the embedding problem as a joint optimization over $N^p,~p = 1, \ldots, m$ rather than $M^p,~p = 1, \ldots, m$. Indeed, the regularization terms can be expressed as 
\begin{equation}\label{eq:Np-reg}
\sum_{p=1}^m \|M^p\|_{\text{HS}}^2 = \trace((N^p)\trans (N^p) K^p).
\end{equation}
The embedding function can now be rewritten as
\[
g(x) = \left( M^p \left(\phi^p(x)\right) \right)_{p=1}^m = \left( N^p K^p_x \right)_{p=1}^m,
\]
and the inner products between embedded points take the form:
\begin{align*}
A_{ij} = \langle g(x_i), g(x_j) \rangle &= \sum_{p=1}^m \left(N^p K^p_i\right)\trans \left(N^p K^p_j \right)\nonumber\\
&= \sum_{p=1}^m (K^p_i)\trans (N^p)\trans (N^p) (K^p_j).
\end{align*}
Similarly, squared Euclidean distance also decomposes by kernel:
\begin{equation}\label{eq:Np-dist}
\|g(x_i) - g(x_j)\|^2 = \sum_{p=1}^m \left(K^p_i - K^p_j\right)\trans (N^p)\trans (N^p) \left(K^p_i - K^p_j\right).
\end{equation}



Finally, since 
the matrices $N^p,~p=1, \ldots,m$ only appear in the form of inner products in (\ref{eq:Np-reg}) and (\ref{eq:Np-dist}), we may instead optimize over PSD matrices $W^p = (N^p)\trans (N^p)$. This renders the regularization terms (\ref{eq:Np-reg}) and distances (\ref{eq:Np-dist}) linear in the optimization variables $W^p$. Extending Algorithm~\ref{alg:kpoe} to this parameterization of $g(\cdot)$ therefore results in an SDP, which is listed as Algorithm~\ref{alg:mkpoe}.  To solve the SDP, we implemented a gradient descent solver, which is described in Appendix~\ref{sec:gradient}.


The class of kernels over which Algorithm~\ref{alg:mkpoe} optimizes can be expressed algebraically as
\begin{equation}
\K_4 = \left\{  \sum_{p=1}^m K^{p} W^p K^{p} ~:~ \forall p,~W^p\succeq 0 \right\}\label{kernelclassconcat}.
\end{equation}
Note that $\K_4$ contains $\K_3$ as a special case when all $W^p$ are positive scalar multiples of each-other.  However, $\K_4$ 
leads to a convex optimization problem, where $\K_3$ does not.  


Table~\ref{tab:kernelcombination} lists the block-matrix formulations of each of the kernel combination rules described in this section.  It is worth noting
that it is certainly valid to first form the unweighted combination kernel $\widehat{K}$ and then use $\K_1$ (Algorithm~\ref{alg:kpoe}) to learn an optimal
projection of the product space.  However, as we will demonstrate in Section~\ref{sec:exp}, our proposed multiple-kernel formulation ($\K_4$) outperforms 
the simple unweighted combination rule in practice.

\begin{table}
\begin{tabular}{ll}
Kernel class & Learned kernel matrix\\
\hline
\hline
$\K_1 = \left\{ \sum_{p,q} K^p W K^q \right\}$ & $\left[K^1 + K^2 + \dots + K^m\right][W]\left[K^1 + K^2 + \dots + K^m\right]$\\
&\\
$\K_2 = \left\{ \sum_{p,q} \mu_p \mu_q K^p W K^q \right\}$ & 
$ 
\left[
    \begin{array}{c}
        K^1\\
        K^2\\
        \vdots\\
        K^m 
    \end{array}
\right]\trans
\left[ 
    \begin{array}{cccc}
        \mu_1^2 W       & \mu_1\mu_2 W  & \cdots    & \mu_1\mu_m W\\
        \mu_2\mu_1 W    & \mu_2^2 W     & \cdots    & \vdots\\
        \vdots          &               & \ddots    & \\
        \mu_m\mu_1 W    &               &           & \mu_m^2 W\\
    \end{array}
\right]
\left[
    \begin{array}{c}
        K^1\\
        K^2\\
        \vdots\\
        K^m 
    \end{array}
\right]
$\\
&\\
$\K_3 = \left\{ \sum_p \mu_p^2 K^p W K^p \right\}$ & 
$ 
\left[
    \begin{array}{c}
        K^1\\
        K^2\\
        \vdots\\
        K^m 
    \end{array}
\right]\trans
\left[ 
    \begin{array}{cccc}
        \mu_1^2 W    & 0             & \cdots    & 0 \\
        0               & \mu_2^2 W  & \cdots    & \vdots\\
        \vdots          &               & \ddots    & \\
        0               &               &           & \mu_m^2 W\\
    \end{array}
\right]
\left[
    \begin{array}{c}
        K^1\\
        K^2\\
        \vdots\\
        K^m 
    \end{array}
\right]
$\\
&\\
$\K_4 = \left\{ \sum_p K^p W^p K^p \right\}$ & 
$ 
\left[
    \begin{array}{c}
        K^1\\
        K^2\\
        \vdots\\
        K^m 
    \end{array}
\right]\trans
\left[ 
    \begin{array}{cccc}
        W^1    & 0             & \cdots    & 0 \\
        0               & W^2  & \cdots    & \vdots\\
        \vdots          &               & \ddots    & \\
        0               &               &           & W^m\\
    \end{array}
\right]
\left[
    \begin{array}{c}
        K^1\\
        K^2\\
        \vdots\\
        K^m 
    \end{array}
\right]
$\\
\end{tabular}
\caption{Block-matrix formulations of metric learning for multiple-kernel formulations ($\K_1$--$\K_4$).  Each $W^p$ is taken to be positive semi-definite.
Note that all sets are equal when there is only one base kernel.
\label{tab:kernelcombination}}
\end{table}

\begin{figure}
\begin{center}
\subfigure[Weighted combination ($\K_2$)]{\includegraphics[height=0.12\textheight]{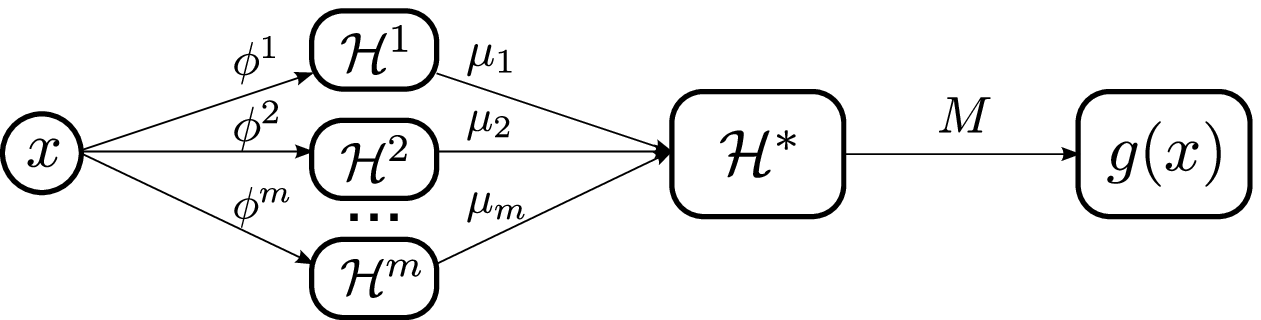}\label{fig:mkl:embedding:early}}\\%
\subfigure[Concatenated projection ($\K_4$)]{\includegraphics[height=0.12\textheight]{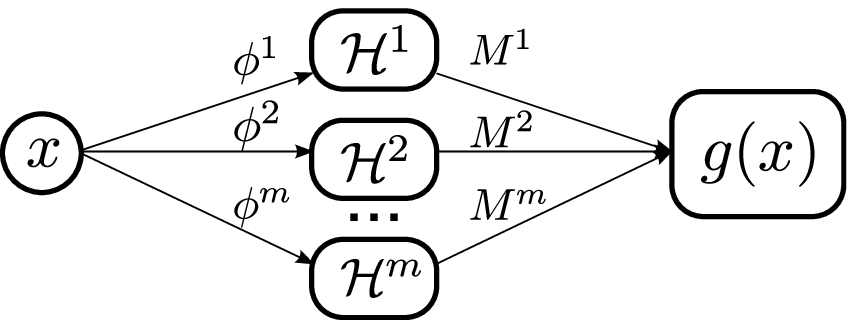}\label{fig:mkl:embedding:late}}%
\end{center}
\caption{Two variants of multiple-kernel embedding.  \subref{fig:mkl:embedding:early} A data point $x\in\X$ is mapped into $m$ feature spaces via 
$\phi^1,\phi^2,\dots,\phi^m$, which are then scaled by $\mu_1,\mu_2,\dots,\mu_m$ to form a weighted feature space $\H^*$, which is subsequently projected 
to the embedding space via $M$.  \subref{fig:mkl:embedding:late} $x$ is first mapped into each kernel's feature space and then its image in each space is 
directly projected into a Euclidean space via the corresponding projections $M^p$.  The projections are jointly optimized to produce the embedding space.\label{fig:mkl:embedding}}
\end{figure}

\begin{algorithm}
\caption{Multiple kernel partial order embedding (MKPOE)}\label{alg:mkpoe}%
\begin{algorithmic}
    \STATE {\bfseries Input:} $n$ objects $\X$,\\partial order $\C$,\\$m$ kernel matrices $K^1,K^2,\dots,K^m$,\\$\beta> 0$
    \STATE {\bfseries Output:} mapping $g:\X\rightarrow \R^{mn}$
    \begin{align*}
        \min_{W^p,\xi} ~~~ &\hfill \sum_{p=1}^m \trace\left(W^pK^p\right) + \frac{\beta}{|\C|} \sum_{\C} \xi_{ijk\ell}\\
        d(x_i,x_j) &\doteq \sum_{p=1}^m \left( K^p_i - K^p_j\right)\trans W^p \left(K^p_i - K^p_j\right)\\
        d(x_i,x_j) + 1 & \leq d(x_k,x_\ell) + \xi_{ijk\ell}\\
        \xi_{ijk\ell} & \geq 0
            &\forall (i,j,k,\ell) \in \C\\
        W^p & \succeq 0 
            &p=1,2,\dots,m\\
    \end{align*}
\end{algorithmic}
\end{algorithm}

\subsection{Diagonal learning}
The MKPOE optimization is formulated as a semi-definite program over $m$ different $n{\times}n$ matrices $W^p$ --- or, as shown in Table~\ref{tab:kernelcombination}, a single $mn{\times}mn$ PSD matrix with a block-diagonal sparsity structure.  Scaling this approach to large data sets
can become problematic, as they require optimizing over multiple high-dimensional PSD matrices.

To cope with larger problems, the optimization problem can be refined to constrain each $W^p$ to the set of diagonal matrices.  If $W^p$ are all
diagonal, positive semi-definiteness is equivalent to non-negativity of the diagonal values (since they are also the eigenvalues of the matrix).  This
allows the constraints $W^p \succeq0$ to be replaced by linear constraints $W^p_{ii} \geq 0$, and the resulting optimization problem is a linear program
(LP), rather than an SDP.  This modification reduces the flexibility of the model, but leads to a much more efficient optimization procedure.

More specifically, our implementation of Algorithm~\ref{alg:mkpoe} operates by alternating gradient descent on $W^p$ and projection onto the feasible set $W^p\succeq0$ (see Appendix~\ref{sec:gradient} for details).  For full matrices, this projection is accomplished by computing the spectral decomposition of each $W^p$, and thresholding the eigenvalues at 0.  For diagonal matrices, this projection is accomplished simply by
\[
W^p_{ii} \mapsto \max\left\{0, W^p_{ii}\right\},
\]
which can be computed in $O(mn)$ time, compared to the $O(mn^3)$ time required to compute $m$ spectral decompositions.  

Restricting $W^{p}$ to be diagonal not only simplifies the problem to linear programming, but carries the added interpretation 
of weighting the contribution of each {(kernel, training point)} pair in the construction of the embedding.  A large value at 
$W^{p}_{ii}$ corresponds to point $i$ being a landmark for the features encoded in $K^{p}$.  Note that each of the formulations listed in
Table~\ref{tab:kernelcombination} has a corresponding diagonal variant, however, as in the full matrix case, only $\K_1$ and $\K_4$ lead to 
convex optimization problems.

\section{Experiments}
\label{sec:exp}

To evaluate our framework for learning multi-modal similarity, we first test the multiple kernel learning formulation on a simple toy taxonomy data set, and then on a real-world data set of musical
perceptual similarity measurements.

\subsection{Toy experiment: Taxonomy embedding}
\label{sec:taxonomy}
For our first experiment, we generated a toy data set from the Amsterdam Library of Object Images (ALOI) data set~\citep{aloi}.  ALOI consists of RGB 
images of 1000 classes of objects against a black background.  Each class corresponds to a single object, and examples are provided of the object under 
varying degrees of out-of-plane rotation.

In our experiment, we first selected 10 object classes, and from each class, sampled 20 examples.  We then constructed an artificial taxonomy over the label
set, as depicted in Figure~\ref{fig:taxonomy}.  Using the taxonomy, we synthesized relative comparisons to span subtrees via their least common ancestor.  
For example, 
\begin{align*}
(Lemon\,\#1,\,Lemon\,\#2,\, Lemon\,\#1,\, Pear\#1),\\
(Lemon\,\#1,\, Pear\#,1,\, Lemon\,\#1,\, Sneaker\#1),
\end{align*}
and so on.  These comparisons are consistent and therefore can be represented as a directed acyclic graph. They are generated so as to avoid redundant, transitive edges in the graph.

For features, we generated five kernel matrices.  The first is a simple linear kernel over the grayscale intensity values of the images, which, roughly
speaking, compares objects by shape.  The other four are Gaussian kernels over histograms in the (background-subtracted) red, green, blue, and 
intensity channels, and these kernels compare objects based on their color or intensity distributions.

We augment this set of kernels with five ``noise'' kernels, each of which was generated by sampling random points from the unit sphere in $\R^3$ and
applying the linear kernel.

\begin{figure}
\centering
\includegraphics[width=0.3\textwidth]{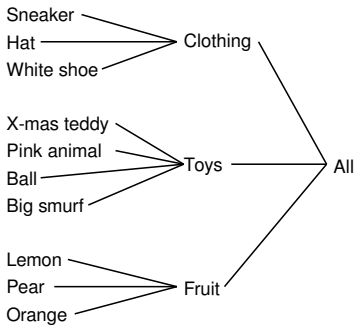}
\caption{The label taxonomy for the experiment in Section~\ref{sec:taxonomy}.\label{fig:taxonomy}}
\end{figure}

The data was partitioned into five 80/20 training and test set splits.  To tune $\beta$, we further split the training set for 5-fold cross-validation, 
and swept over $\beta\in\{10^{-2},10^{-1},\dots,10^{6}\}$.  For each fold, we learned a diagonally-constrained embedding with Algorithm~\ref{alg:mkpoe}, 
using the subset of relative comparisons $(i,j,k,\ell)$ with $i, j, k$ and $\ell$ restricted to the training set.  
After learning the embedding, the held out data (validation or test) was mapped into 
the space, and the accuracy of the embedding was determined by counting the fraction of correctly predicted relative comparisons.  In the validation and 
test sets, comparisons were processed to only include comparisons of the form $(i,j,i,k)$ where $i$ belongs to the validation (or test) set, and $j$ and $k$ 
belong to the training set.

We repeat this experiment for each base kernel individually (\ie, optimizing over $\K_1$ with a single base kernel), as well as the unweighted sum kernel ($\K_1$ with all base kernels), and finally MKPOE ($\K_4$ with all base kernels).  The results are averaged over all training/test splits, and collected in Table~\ref{tab:taxonomy}.  For comparison purposes, we include the prediction accuracy achieved by computing distances in each kernel's native space before
learning.  In each case, the optimized space indeed achieves higher accuracy than the corresponding native space.  (Of course, the random noise kernels
still predict randomly after optimization.)

As illustrated in Table~\ref{tab:taxonomy}\subref{tab:taxonomy:mkl}, taking the unweighted combination of kernels significantly degrades performance (relative to the best kernel) both in the native space (0.718 accuracy versus 0.862 for the linear kernel) and the optimized sum-kernel space (0.861 accuracy for $\K_1$ versus 0.951 for the linear kernel), \ie, the unweighted sum kernel optimized by Algorithm~\ref{alg:kpoe}.
However, MKPOE ($\K_4$) correctly identifies and omits the random noise kernels by assigning them negligible weight, and achieves higher accuracy (0.984) than any of the single kernels (0.951 for the 
linear kernel, after learning).  

\begin{table}
    \centering
\subtable[]{
    \begin{tabular}{lrr}
        \multirow{2}{*}{Base Kernel}                      & \multicolumn{2}{c}{Accuracy}\\
        \cline{2-3}
                        & Native            & $\K_1$\\
        \hline
        Linear          &   \textbf{0.862}  &  \textbf{0.951} \\ 
        Red             &   0.623           &  0.694 \\ 
        Green           &   0.588           &  0.716 \\ 
        Blue            &   0.747           &  0.825 \\ 
        Intensity       &   0.659           &  0.792 \\ 
        Random 1        &   0.484           &  0.506 \\ 
        Random 2        &   0.494           &  0.504 \\ 
        Random 3        &   0.475           &  0.505 \\ 
        Random 4        &   0.495           &  0.510 \\ 
        Random 5        &   0.525           &  0.522 
    \end{tabular}
\label{tab:taxonomy:base}
}%
\hfill\subtable[]{
    \begin{tabular}{lrrr}
        & \multicolumn{3}{c}{Accuracy}\\
        \cline{2-4}
                                        & Native    & $\K_1$ & $\K_4$\\
        \hline
        MKL &   0.718  &  0.861 & \textbf{0.984} \\ 
    \end{tabular}
\label{tab:taxonomy:mkl}
}
    \caption{Average test set accuracy for the experiment of Section~\ref{sec:taxonomy}.  \subref{tab:taxonomy:base} Accuracy is computed by counting the fraction of correctly predicted
    relative comparisons in the native space of each base kernel, and then in the space produced by KPOE ($\K_1$ with a single base kernel).  \subref{tab:taxonomy:mkl} The unweighted combination
    of kernels significantly degrades performance, both in the native space, and the learned space ($\K_1$).  MKPOE ($\K_4$) correctly rejects the random
    kernels, and significantly outperforms the unweighted combination and the single best kernel.  \label{tab:taxonomy}}
\end{table}

\subsection{Musical artist similarity}
To test our framework on a real data set, we applied the MKPOE algorithm to the task of learning a similarity function between musical artists.  The 
artist similarity problem is motivated by several real-world applications, including recommendation and playlist-generation for online radio.  Because
artists may be represented by a wide variety of different features (\eg, tags, acoustic features, social data), such applications can benefit greatly from
an optimally integrated similarity metric.

The training data is derived from the \emph{aset400} corpus of~\cite{ellis02}, which consists of 412 popular musicians, and 16385 relative
comparisons of the form $(i,j,i,k)$.  Relative comparisons were acquired from human test subjects through a web survey; subjects were presented with a query
artist ($i$), and asked to choose what they believe to be the most similar artist ($j$) from a list of 10 candidates.  From each single response, 9 relative
comparisons are synthesized, indicating that $j$ is more similar to $i$ than the remaining 9 artists ($k$) which were not chosen.

Our experiments here replicate and extend previous work on this data set~\citep{mcfee09_hesas}.  In the remainder of this section, we will first give an overview of the various types of features used to characterize each artist in Section~\ref{sec:exp:aset:feature}.  We will then discuss the
experimental procedure in more detail in Section~\ref{sec:exp:aset:experiment}.  The MKL embedding results are presented in Section~\ref{sec:exp:aset:embedding}, and are followed by an experiment detailing
the efficacy of our constraint graph processing approach in Section~\ref{sec:exp:aset:graph}.

\subsubsection{Features}
\label{sec:exp:aset:feature}
We construct five base kernels over the data, incorporating acoustic, semantic, and social views of the artists.  

\begin{itemize}
    \item \textbf{MFCC}: for each artist, we collected between 1 and 10 songs (mean 4).  For each song, we extracted a short clip consisting of 10000 half-overlapping 23ms windows. For each window, we 
    computed the first 13 Mel Frequency Cepstral Coefficients (MFCCs)~\citep{davis90}, as well as their first and second instantaneous derivatives.
This results in a sequence of 39-dimensional vectors (delta-MFCCs) for each song.  Each artist $i$ was then summarized by a Gaussian mixture model (GMM) $p_i$ over delta-MFCCs
    extracted from the corresponding songs.  Each GMM has 8 components and diagonal covariance matrices.  Finally, the kernel between artists $i$ and $j$ 
    is the probability product kernel~\citep{ppk} between their corresponding delta-MFCC distributions $p_i,p_j$:
    \[
        K^{\text{mfcc}}_{ij} = \int \sqrt{p_i(x) p_j(x)}\,dx.
    \]

    \item \textbf{Auto-tags}: Using the MFCC features described above, we applied the automatic tagging algorithm of~\cite{Turnbull_SemanticAudio},
    which for each song yields a multinomial distribution over a set $T$ of 149 musically-relevant tag words (\emph{auto-tags}).  Artist-level 
    tag distributions $q_i$ were formed by averaging model parameters (\ie, tag probabilities) across all of the songs of artist $i$.  The kernel between artists 
    $i$ and $j$ for auto-tags is a radial basis function applied to the $\chi^2$-distance between the multinomial distributions $q_i$ and $q_j$:
    \[
        K^{\text{at}}_{ij} = \exp\left(-\sigma \sum_{t\in T} \frac{\left(q_i(t) - q_j(t)\right)^2}{q_i(t) + q_j(t)}\right).
    \]
    In these experiments, we fixed $\sigma=256$.

    \item \textbf{Social tags}: For each artist, we collected the top 100 most frequently used tag words from Last.fm,\footnote{http://last.fm} a 
    social music website which allows users to label songs or artists with arbitrary tag words or \emph{social tags}.  After stemming and stop-word removal, this
    results in a vocabulary of 7737 tag words.  Each artist is then represented by a bag-of-words vector in $\R^{7737}$, and processed by TF-IDF.  The 
    kernel between artists for social tags is the cosine similarity (linear kernel) between TF-IDF vectors.
    
    \item \textbf{Biography}: Last.fm also provides textual descriptions of artists in the form of user-contributed biographies.  We collected biographies
    for each artist in the \emph{aset400} data set, and after stemming and stop-word removal, we arrived at a vocabulary of 16753 biography words.  As with social tags, the
    kernel between artists is the cosine similarity between TF-IDF bag-of-words vectors.
    
    \item \textbf{Collaborative filtering}: \cite{Celma:Thesis2008} collected collaborative filtering data from Last.fm in the form of a bipartite graph 
    over users and artists, where each user is associated with the artists in her listening history.  We filtered this data down to include only the aset400 
    artists, of which all but 5 were found in the collaborative filtering graph.  The resulting graph has 336527 users and 407 artists, and is equivalently 
    represented by a binary matrix where each row $i$ corresponds to an artist, and each column $j$ corresponds to a user.  The $ij$ entry of this matrix is 1 if we
    observe a user-artist association, and 0 otherwise.
    The kernel between artists in this view 
    is the cosine of the angle between corresponding rows in the matrix, which can be interpreted as counting the amount of overlap between the sets of users listening to each artist
    and normalizing for overall artist popularity.  For the 5 artists not found in the graph, we fill in the corresponding rows and columns of the kernel
    matrix with the identity matrix.

\end{itemize}

\subsubsection{Experimental procedure}
\label{sec:exp:aset:experiment}






The data set was split into 330 training and 82 test artists.  
Given the inherent ambiguity in the task and the format of the survey, there is a great deal of conflicting information in the survey responses.  
To obtain a more accurate and internally coherent set of training comparisons, directly contradictory comparisons (\eg, $(i,j,i,k)$ and $(i,k,i,j)$) 
are removed from the training set, reducing the set from 7915 to 6583 relative comparisons.  
The training set is further cleaned by finding an acyclic subset of comparisons and taking its transitive reduction, resulting in a minimal partial
order with 4401 comparisons.

To evaluate the performance of an embedding learned from the training data, we apply it to the test data, and then
measure accuracy by counting the fraction of similarity measurements $(i,j,i,k)$ correctly predicted by distance in the embedding space, where $i$ belongs to the test set, and $j$ and $k$ belong to the training set.
This setup can be viewed as simulating a query (by-example) $i$ and ranking the responses $j,k$ from the training set.
To gain a more accurate view of the quality of the embedding, the test set was also pruned to remove directly contradictory measurements.  This
reduces the test set from 2095 to 1753 comparisons.  No further processing is applied to test measurements, and we note that the test set is not internally 
consistent, so perfect accuracy is not achievable.

For each experiment, the optimal $\beta$ is chosen from $\{10^{-2}, 10^{-1}, \dots, 10^7\}$ by 10-fold cross-validation, \ie, 
repeating the test procedure above on splits within the training set.  Once $\beta$ is chosen, an embedding is learned 
with the entire training set, and then evaluated on the test set.

\subsubsection{Embedding results}
\label{sec:exp:aset:embedding}

For each base kernel, we evaluate the test-set performance in the native space (\ie, by distances calculated directly from the entries of the kernel
matrix), and by learned metrics, both diagonal and full (optimizing over $\K_1$ with a single base kernel).  Table~\ref{tab:result:single} lists the results. In all cases, we observe significant improvements in accuracy over the native space.  In all but one case, full-matrix embeddings significantly 
outperform diagonally-constrained embeddings.
\begin{table}
    \centering
    \begin{tabular}{lrrr}
        \multirow{2}{*}{Kernel}                      & \multicolumn{3}{c}{Accuracy}\\
        \cline{2-4}
                                        & Native    & $\K_1$ (diagonal) & $\K_1$ (full)\\
        \hline
        MFCC                            & 0.464     & 0.593 & 0.590\\
        Auto-tags (AT)                  & 0.559     & 0.568 & 0.594\\
        Social tags (ST)                & 0.752     & 0.773 & \textbf{0.796}\\
        Biography (Bio)                 & 0.611     & 0.629 & 0.760\\
        Collaborative filter (CF)       & 0.704     & 0.655 & 0.776\\
    \end{tabular}
\caption{aset400 embedding results for each of the base kernels.  Accuracy is computed in each kernel's native feature space, as well as the space produced
by applying Algorithm~\ref{alg:kpoe} (\ie, optimizing over $\K_1$ with a single kernel) with either the diagonal or full-matrix 
formulation.\label{tab:result:single}}
\end{table}

We then repeated the experiment by examining different groupings of base kernels: acoustic (MFCC and Auto-tags), semantic (Social tags and Bio), social
(Collaborative filter), and combinations of the groups.  The different sets of kernels were combined by Algorithm~\ref{alg:mkpoe} (optimizing over
$\K_4$).  The results are listed in Table~\ref{tab:result:mkl}.  For comparison purposes, we also include the unweighted sum of all base kernels (listed in
the \emph{Native} column).

In all cases, MKPOE improves over the unweighted combination of base kernels.  Moreover, many combinations outperform 
the single best kernel (ST), and the algorithm is generally robust in the presence of poorly-performing distractor kernels (MFCC and AT).  Note that the
poor performance of MFCC and AT kernels may be expected, as they derive from song-level rather than artist-level features, whereas ST provides high-level semantic descriptions which
are generally more homogeneous across the songs of an artist, and Bio and CF are directly constructed at the artist level.  For comparison purposes, we also trained a metric over all kernels with 
$\K_1$ (Algorithm~\ref{alg:kpoe}), and achieve 0.711 (diagonal) and 0.764 (full): significantly worse than the $\K_4$ results.

Figure~\ref{fig:diagmkl} illustrates the weights learned by Algorithm~\ref{alg:mkpoe} using all five kernels and diagonally-constrained $W^p$ matrices.  Note
that the learned metrics are both sparse (many 0 weights) and non-uniform across different kernels.  In particular, the (lowest-performing) MFCC kernel is 
eliminated by the algorithm, and the majority of the weight is assigned to the (highest-performing) social tag (ST) kernel.

A t-SNE~\citep{tsne} visualization of the space produced by MKPOE is illustrated in Figure~\ref{fig:embedding}.  The embedding captures a great deal of 
high-level genre structure in low dimensions: for example, the \emph{classic rock} and \emph{metal} genres lie at the opposite end of the space from 
\emph{pop} and \emph{hip-hop}.


\begin{table}
    \centering
    \begin{tabular}{lrrr}
        \multirow{2}{*}{Base kernels}                & \multicolumn{3}{c}{Accuracy}\\
        \cline{2-4}
                                    & Native    & $\K_4$ (diagonal)     & $\K_4$(full)\\
        \hline
        MFCC + AT                   & 0.521     & 0.589                 & 0.602\\
        ST + Bio                    & 0.760     & 0.786                 & 0.811\\
        MFCC + AT + CF              & 0.580     & 0.671                 & 0.719\\
        ST + Bio + CF               & 0.777     & 0.782                 & 0.806\\
        MFCC + AT + ST + Bio        & 0.709     & 0.788                 & 0.801\\
        All                         & 0.732     & 0.779                 & 0.801\\
        \\
    \end{tabular}
\caption{aset400 embedding results with multiple kernel learning: the learned metrics are optimized over 
$\K_4$ by Algorithm~\ref{alg:mkpoe}.  \emph{Native} corresponds to distances calculated according to the unweighted sum of base kernels.  
\label{tab:result:mkl}}
\end{table}


\begin{figure}
\begin{center}
\includegraphics[width=\textwidth]{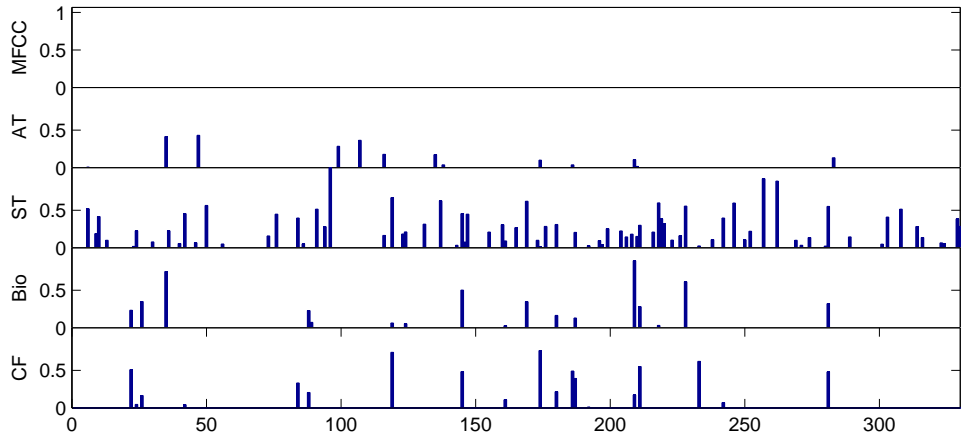}%
\end{center}
\caption{The weighting learned by Algorithm~\ref{alg:mkpoe} using all five kernels and diagonal $W^p$.  
Each bar plot contains the diagonal of the corresponding kernel's learned metric.
The horizontal axis corresponds the index of the
training set, and the vertical axis corresponds to the learned weight in each kernel space.  
\label{fig:diagmkl}}
\end{figure}

\begin{figure}
\centering
\subfigure[]{\includegraphics[height=0.4\textheight]{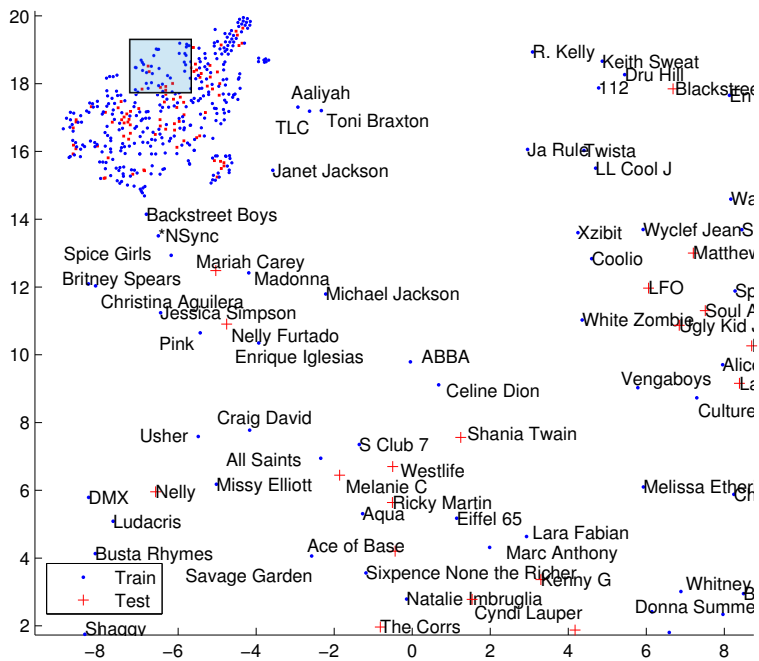}\label{fig:embedding:pop}}
\subfigure[]{\includegraphics[width=0.475\textheight]{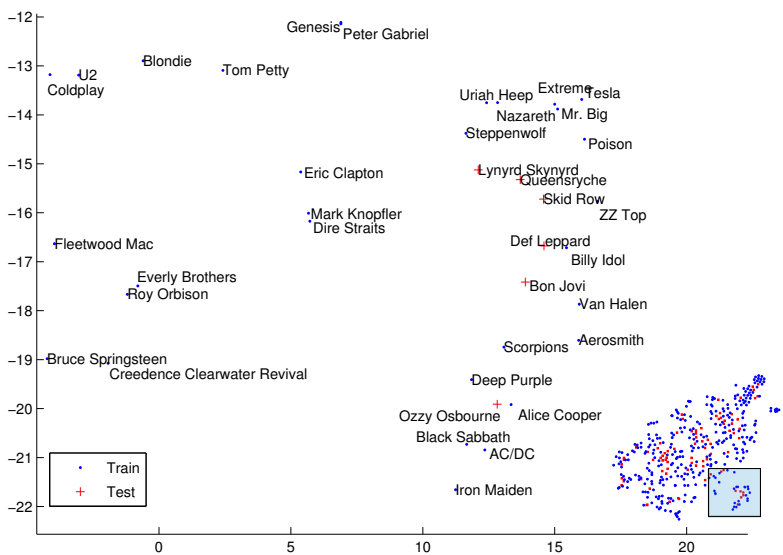}\label{fig:embedding:rock}}
\caption{t-SNE visualizations of an embedding of aset400 produced by MKPOE.  The embedding is constructed by optimizing over $\K_4$ with all five base kernels.  The two clusters shown roughly corrsepond to \subref{fig:embedding:pop} pop/hip-hop, and \subref{fig:embedding:rock} classic rock/metal genres.  Out-of-sample points are indicated by a red +.\label{fig:embedding}}
\end{figure}

\subsubsection{Graph processing results}
\label{sec:exp:aset:graph}



To evaluate the effects of processing the constraint set for consistency and redundancy, we repeat the experiment of the previous section with different levels of processing applied to $\C$.  Here, we focus on the Biography kernel, since it exhibits the largest gap in performance between the native and 
learned spaces.

As a baseline, we first consider the full set of similarity measurements as provided by human judgements, including all inconsistencies.  In the 80-20 split, there are 7915 total
training measurements.
To first deal with what appear to be the most eggregious inconsistencies, we prune all directly inconsistent training measurements; \ie, whenever $(i,j,i,k)$ and $(i,k,i,j)$ both appear, both are removed.\footnote{A more sophisticated approach could be used here, \eg, majority voting, provided there is sufficient over-sampling of comparisons in the data.}
This variation results in 6583 training measurements, and while they are not wholly consistent, the worst violators have been pruned.
Finally, we consider the fully processed case by finding a maximal consistent subset (partial order) of $\C$ and removing all redundancies, resulting in a partial order with 4401 measurements.  

Using each of these variants of the training set, we test the embedding algorithm with both diagonal and full-matrix formulations.  The results are
presented in Table~\ref{tab:graphprocessing}.  Each level of graph processing results in a small improvement in the accuracy of the learned space, and provides 
substantial reductions in computational overhead at each step of the optimization procedure for Algorithm~\ref{alg:kpoe}.

\begin{table}
\centering
\begin{tabular}{lrr}
            & \multicolumn{2}{c}{Accuracy}\\
\cline{2-3}
$\C$        & Diagonal  &   Full\\
\hline
Full        &   0.604   &   0.754\\
Length-2    &   0.621   &   0.756\\
Processed   &   0.629   &   0.760\\
\end{tabular}
\caption{aset400 embedding results (Biography kernel) for three possible refinements of the constraint set.  \emph{Full} includes all similarity measurements, with no pruning for 
consistency or redundancy.  \emph{Length-2} removes all length-2 cycles (\ie, $(i,j,k,\ell)$ and $(k,\ell,i,j)$).  \emph{Processed} finds an approximate
maximal consistent subset, and removes redundant constraints.\label{tab:graphprocessing}}
\end{table}

\section{Hardness of dimensionality reduction}
\label{sec:nphard}
The algorithms given in Sections~\ref{sec:parametric} and~\ref{sec:mkl} attempt to produce low-dimensional solutions by regularizing $W$, which can be seen
as a convex approximation to the rank of the embedding.  
In general, because rank constraints are not convex, convex optimization techniques cannot efficiently minimize dimensionality.  This does not necessarily
imply other techniques could not work.
So, it is natural to ask if exact solutions of minimal dimensionality 
can be found efficiently, particularly in the multidimensional scaling scenario, \ie, when $K=I$ (Section~\ref{sec:parametric:gnmds}).

As a special case, one may wonder if any 
instance $(\X,\C)$ can be satisfied in $\R^1$.  As Figure~\ref{fig:square} demonstrates, not all instances can be realized in one dimension.
Even more, we show that it is NP-Complete to decide if a given $\C$ can be satisfied in $\R^1$.  Given an embedding,
it can be verified in polynomial time whether $\C$ is satisfied or not by simply computing the distances between all pairs and checking 
each comparison in $\C$, so the decision problem is in NP.
It remains to show that the $\R^1$ partial order embedding problem (hereafter referred to as \emph{1-POE\/}) is NP-Hard.  We reduce from the \emph{Betweenness} problem~\citep{opatrny79}, which is known to be NP-complete.

\begin{defn}[Betweenness]
Given a finite set $Z$ and a collection $T$ of ordered triples $(a,b,c)$ of distinct elements from $Z$, is there a one-to-one function
$f:Z\rightarrow\R$ such that for each $(a,b,c)\in~T$, either $f(a)<f(b)<f(c)$ or $f(c)<f(b)<f(a)$?
\end{defn}

\begin{theorem} 1-POE is NP-Hard.\label{thm:nphard} \end{theorem}
\begin{proof}
Let $(Z,T)$ be an instance of Betweenness.  Let $\X = Z$, and for each $(a,b,c)\in~T$, introduce constraints $(a,b,a,c)$ and $(b,c,a,c)$ to $\C$.  Since Euclidean distance in $\R^1$ is simply line distance, these constraints force $g(b)$ to lie between $g(a)$ and $g(c)$.  Therefore, the original 
instance $(Z,T) \in \text{Betweenness}$ if and only if the new instance $(\X,\C) \in \text{1-POE}$.  Since Betweenness is NP-Hard, 1-POE is NP-Hard as well.
\end{proof}

\begin{figure}
\begin{center}
\subfigure[]{\includegraphics[width=0.25\textwidth]{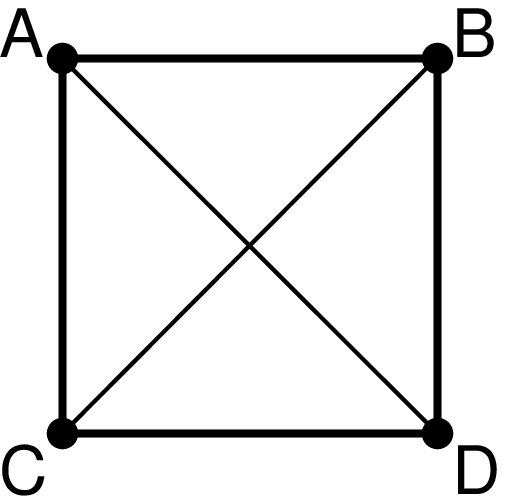}\label{fig:square:left}}%
\hspace{0.25\textwidth}\subfigure[]{\includegraphics[width=0.25\textwidth]{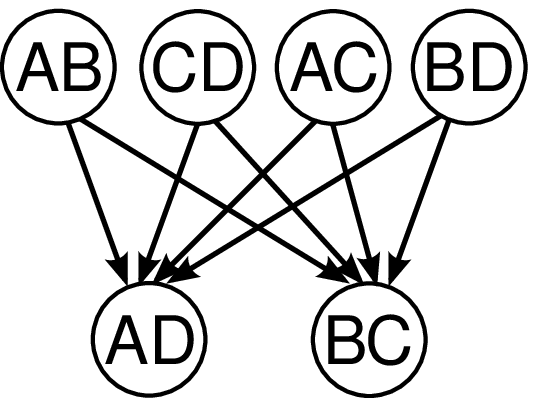}\label{fig:square:right}}%
\end{center}
\caption{\subref{fig:square:left} The vertices of a square in $\R^2$.  \subref{fig:square:right} The partial order over distances induced by the 
square: each side is less than each diagonal.  This constraint set cannot be satisfied in $\R^1$.\label{fig:square}}
\end{figure}

Since 1-POE can be reduced to the general optimization problem of finding an embedding of minimal dimensionality, we can conclude that dimensionality
reduction subject to partial order constraints is also NP-Hard.

\section{Conclusion}
\label{sec:conclusion}
We have demonstrated a novel method for optimally integrating heterogeneous data to conform to measurements of perceptual similarity.  By interpreting a 
collection of relative similarity comparisons as a directed graph over pairs, we are able to apply graph-theoretic techniques to isolate and prune 
inconsistencies in the training set 
and reduce computational overhead by eliminating redundant constraints in the optimization procedure.

Our multiple-kernel formulation offers a principled way to integrate multiple feature modalities into a unified similarity space.  Our formulation carries
the intuitive geometric interpretation of concatenated projections, and results in a semidefinite program.  By incorporating diagonal
constraints as well, we are able to reduce the computational complexity of the algorithm, and learn a model which is both flexible --- only using kernels in the portions of the space where they are informative ---
and interpretable --- each diagonal weight corresponds to the contribution to the optimized space due to a single point within a single feature space.
Table~\ref{tab:kernelcombination} provides a unified perspective of multiple kernel learning formulations for embedding problems, but it is clearly not complete.  It will
be the subject of future work to explore and compare alternative generalizations and restrictions of the formulations presented here.

\appendix
\section{Embeddability of partial orders}
\label{sec:diameter}
\label{sec:diameterproof}
In this appendix, we prove that any set $\X$ with a partial order over distances $\C$ can be embedded into $\R^{n}$ while satisfying all distance comparisons.

In the special case where $\C$ is a total ordering over all pairs (\ie,\ a chain graph), the problem reduces to non-metric multidimensional scaling~\citep{kruskal64}, and a constraint-satisfying embedding can always be found by the constant-shift embedding algorithm of~\cite{roth03}.  In general, $\C$ is not a total order, but a $\C$-respecting embedding can always be produced by reducing the partial order to a (weak) total order by topologically sorting the graph 
(see Algorithm~\ref{alg:totalorder}).

\begin{algorithm}[tb]
\caption{Na{\" \i}ve total order construction}
\label{alg:totalorder}
\begin{algorithmic}
    \STATE {\bfseries Input:} objects $\X$, partial order $\C$
    \STATE {\bfseries Output:} symmetric dissimilarity matrix $\Delta \in \R^{n\times n}$
    \STATE
    \FOR{each $i$ in $1\dots n$}
        \STATE $\Delta_{ii} \leftarrow 0$
    \ENDFOR
    \FOR{each $(k,\ell)$ in topological order}
        \IF{in-degree$(k,\ell) = 0$}
            \STATE $\Delta_{k\ell},\Delta_{\ell k} \leftarrow 1$
        \ELSE
            \STATE $\Delta_{k\ell},\Delta_{\ell k} \leftarrow \displaystyle\max_{(i,j,k,\ell) \in \C} \Delta_{ij} + 1$
        \ENDIF
    \ENDFOR
\end{algorithmic}
\end{algorithm}


Let $\Delta$ be the dissimilarity matrix produced by Algorithm~\ref{alg:totalorder} on an instance $(\X,\C)$.  An embedding can be found by first applying 
classical multidimensional scaling (MDS)~\citep{cox94} to $\Delta$:
\begin{equation}
A = -\frac{1}{2} H\Delta H,\label{eq:mds}
\end{equation}
where ${H = I - \frac{1}{n} \one \one\trans}$ is the ${n\times n}$ centering matrix, and $\one$ is a vector of 1s.  Shifting the spectrum of $A$ yields
\begin{equation}
A - \lambda_n(A)I = \widehat{A}~\succeq~0,\label{eq:shift}
\end{equation}
where $\lambda_n(A)$ is the minimum eigenvalue of $A$.  The embedding $g$ can be found by decomposing $\widehat{A} = V\widehat{\Lambda} V\trans$, so that 
$g(x_i)$ is the $i\th$ column of ${{}\widehat{\Lambda}}^{1/2}V\trans$; this is the solution constructed by the 
constant-shift embedding non-metric MDS algorithm of~\cite{roth03}.

Applying this transformation to $A$ affects distances by
\begin{align*}
\|g(x_i) - g(x_j)\|^2 = \widehat{A}_{ii} + \widehat{A}_{jj} - 2\widehat{A}_{ij} &= (A_{ii} - \lambda_n) + (A_{jj} - \lambda_n) - 2A_{ij}\\
&= A_{ii} + A_{jj} - 2A_{ij} - 2\lambda_n.
\end{align*}
Since adding a constant ($-2\lambda_n$) preserves the ordering of distances, the total order (and hence $\C$) is preserved by this transformation.  Thus,
for any instance $(\X,\C)$, an embedding can be found in $\R^{n-1}$.

\section{Solver}
\label{sec:gradient}
Our implementation of Algorithm~\ref{alg:mkpoe} is based on a simple projected (sub)gradient descent.  To simplify exposition, we show the derivation 
of the single-kernel SDP version of the algorithm (Algorithm \ref{alg:kpoe}) with unit margins.  (It is straightforward to extend the derivation to the multiple-kernel and LP settings.)

We first observe that a kernel matrix column $K_i$ can be expressed as $K\trans e_i$ where $e_i$ is the $i\th$ standard basis vector.  We can then 
denote the distance calculations in terms of Frobenius inner products:
\begin{align*}
d(x_i,x_j) &= (K_i - K_j)\trans W (K_i - K_j)\\
&= (e_i - e_j)\trans K W K (e_i - e_j)\\
&= \trace(KWK(e_i - e_j)(e_i - e_j)\trans) = \trace(WKE_{ij}K)\\
&= \left\langle W, KE_{ij}K\right\rangle_F,
\end{align*}
where $E_{ij} = (e_i - e_j)(e_i - e_j)\trans$.

A margin constraint $(i,j,k,\ell)$ can now be expressed as:
\begin{align*}
&           & d(x_i,x_j) + 1 &\leq d(x_k,x_\ell) + \xi_{ijk\ell}\\
&\Rightarrow& \left\langle W, KE_{ij}K\right\rangle_F + 1 &\leq \left\langle W, KE_{k\ell}K\right\rangle_F + \xi_{ijk\ell}\\
&\Rightarrow& \xi_{ijk\ell} &\geq 1 + \left\langle W, K(E_{ij} - E_{k\ell})K\right\rangle_F.
\end{align*}

The slack variables $\xi_{ijk\ell}$ can be eliminated from the program by rewriting the objective in terms of the hinge loss $h(\cdot)$ over the constraints:
\[
\min_{W\succeq0} f(W) \text{ where } f(W) = \trace(WK) + \frac{\beta}{|\C|}\sum_\C h\left( 1 + \left\langle W, K(E_{ij} - E_{k\ell})K\right\rangle_F\right).
\]
The gradient $\nabla f$ has two components: one due to regularization, and one due to the hinge loss.  The gradient due to
regularization is simply $K$.  The loss term decomposes linearly, and for each $(i,j,k,\ell)\in\C$, a subgradient direction can be defined:
\begin{equation}
\frac{\partial}{\partial W} h\left( 1 + d(x_i,x_j) - d(x_k,x_\ell)\right) = \begin{cases}
0 & d(x_i,x_j) +1 \leq d(x_k,x_\ell)\\
K(E_{ij} - E_{k\ell})K & \text{otherwise}.
\end{cases}\label{eq:partial}
\end{equation}
Rather than computing each gradient direction independently, we observe that each violated constraint contributes a matrix of the form 
    $K(E_{ij} - E_{k\ell})K$.  By linearity, we can collect all $(E_{ij}-E_{k\ell})$ terms and then pre- and post-multiply by $K$ to obtain a more efficient
calculation of $\nabla f$:
\begin{align*}
\frac{\partial}{\partial W}f &= K + K \left(\sum_{(i,j,k,\ell)\in\overline{\C}} E_{ij}-E_{k\ell} \right) K,
\end{align*}
where $\overline{\C}$ is the set of all currently violated constraints.

After each gradient step $W\mapsto W - \alpha \nabla f$, the updated $W$ is projected back onto the set of positive semidefinite matrices by computing its
spectral decomposition and thresholding the eigenvalues by $\lambda_i\mapsto\max(0,\lambda_i)$.

To extend this derivation to the multiple-kernel case (Algorithm~\ref{alg:mkpoe}), we can define 
\[
d(x_i,x_j) \doteq \sum_{p=1}^m d^p(x_i,x_j),
\]
and exploit linearity to compute each partial derivative $\partial/\partial W^p$ independently.

For the diagonally-constrained case, it suffices to substitute
\[
K(E_{ij}-E_{k\ell})K \quad \mapsto \quad \diag(K(E_{ij}-E_{k\ell})K)
\]
in Equation~\ref{eq:partial}.  After each gradient step in the diagonal case, the PSD constraint on $W$ can be enforced by the projection $W_{ii}\mapsto \max(0,W_{ii})$.

\section{Relationship to AUC}
\label{sec:gauc}
In this appendix, we formalize the connection between partial orders over distances and query-by-example ranking.  Recall that Algorithm~\ref{alg:lpoe}
minimizes the loss $\nicefrac{1}{|\C|}\sum_\C \xi_{ijk\ell}$, where each $\xi_{ijk\ell}\geq0$ is a slack variable associated with a margin constraint
\[
d(i,j) + 1 \leq d(k,\ell) + \xi_{ijk\ell}.
\]

As noted by~\cite{schultz04}, the fraction of relative comparisons satisfied by an embedding $g$ is closely related to the area under the receiver operating characteristic 
curve (AUC).  To make this connection precise, consider the following information retrieval problem.  For each point $x_i \in \X$, we are given a partition of 
$\X\setminus\{i\}$:
\begin{align*}
\X_i^+ &= \{ x_j ~:~ x_j\in\X \text{ relevant for } x_i\},\text{ and }\\
\X_i^- &= \{ x_k ~:~ x_k\in\X \text { irrelevant for } x_i\}.
\end{align*}
If we embed each $x_i\in\X$ into a Euclidean space, we can then rank the rest of the data $\X\setminus\{x_i\}$ by increasing distance from $x_i$.  Truncating this
ranked list at the top $\tau$ elements (\ie, closest $\tau$ points to $x_i$) will return a certain fraction of relevant points (true positives), and irrelevant
points (false positives).
Averaging over all values of $\tau$ defines the familiar AUC score, which can be compactly expressed
as:
\[
\text{AUC}(x_i|g) = 
\frac{1}{|\X_i^+|\cdot|\X_i^-|} \sum_{(x_j,x_k) \in \X_i^+\times\X_i^-} \I\left[\|g(x_i) - g(x_j)\| < \|g(x_i) - g(x_k)\|\right].
\]

Intuitively, AUC can be interpreted as an average over all pairs $(x_j,x_k) \in \X_i^+\times\X_i^-$ of the number of times a $x_i$ was mapped closer to a 
relevant point $x_j$ than an irrelevant point $x_k$.
This in turn can be conveniently expressed by a set of relative comparisons for each $x_i\in\X$: 
\[
    \forall (x_j,x_k) \in \X_i^+\times\X_i^-~:~ (i,j,i,k).
\]
An embedding which satisfies a complete set of constraints of this form will receive an AUC score of 1, since every relevant point must be closer to $x_i$
than every irrelevant point.

Now, returning to the more general setting, we do not assume binary relevance scores or complete observations of relevance for all pairs of points.
However, we can define the generalized AUC score (GAUC) as simply the average number of correctly ordered pairs (equivalently, satisfied constraints) given a set of relative
comparisons:
\begin{equation}
\text{GAUC}(g) = \frac{1}{|\C|} \sum_{(i,j,k,\ell)\in\C} \I\left[\|g(x_i)-g(x_j)\| < \|g(x_k)-g(x_\ell)\|\right].\label{eq:gauc}
\end{equation}
Like AUC, GAUC is bounded between 0 and 1, and the two scores coincide exactly in the previously described ranking problem.  A corresponding loss function can be defined by reversing the order of the
inequality, \ie,
\[
L_\text{GAUC}(g) = \frac{1}{|\C|} \sum_{(i,j,k,\ell)\in\C} \I\left[\|g(x_i)-g(x_j)\| \geq \|g(x_k)-g(x_\ell)\|\right].
\]
Note that $L_\text{GAUC}$ takes the form of a sum over indicators, and can be interpreted as the average 0/1-loss over $\C$.  This function is clearly not convex in $g$, and is therefore difficult to 
optimize.  Algorithms~\ref{alg:lpoe},~\ref{alg:kpoe} and \ref{alg:mkpoe} instead optimize a convex upper bound on $L_\text{GAUC}$ by replacing indicators by the hinge loss:
\[
h(x) = \begin{cases}
0 & x \leq 0\\
x & x > 0
\end{cases}.
\]
As in SVM, this is accomplished by introducing a unit margin and slack variable $\xi_{ijk\ell}$ for each $(i,j,k,\ell)\in\C$, and minimizing $\nicefrac{1}{|\C|}\sum_\C\xi_{ijk\ell}$.

\bibliography{refs}
\end{document}